\documentclass{article}

\usepackage[utf8]{inputenc} 
\usepackage[T1]{fontenc}    
\usepackage{hyperref}       
\usepackage{url}            
\usepackage{booktabs}       
\usepackage{amsfonts}       
\usepackage{nicefrac}       
\usepackage{microtype}      
\usepackage{xcolor}         

\usepackage{amsmath}
\usepackage{amssymb}
\usepackage{mathtools}
\usepackage{amsthm}
\usepackage[final]{changes}
\usepackage[numbers, compress]{natbib}
\usepackage{authblk}
\theoremstyle{plain}
\newtheorem{theorem}{Theorem}[section]
\newtheorem{proposition}[theorem]{Proposition}
\newtheorem{lemma}[theorem]{Lemma}

\theoremstyle{definition}

\theoremstyle{remark}
\newtheorem{remark}[theorem]{Remark}

\usepackage{braket}
\usepackage{interval}
\newcommand{\braketH}[1]{\braket{#1}_{\text{H}}}
\newcommand{\BraketH}[1]{\Braket{#1}_{\text{H}}}
\DeclareMathOperator{\tr}{tr}
\DeclareMathOperator*{\argmax}{argmax}
\DeclareMathOperator*{\argmin}{argmin}

\begin{document}

\title{Maximum-Likelihood Quantum State Tomography by Soft-Bayes}

\author[1]{Chien-Ming Lin}
\author[1]{Yu-Ming Hsu}
\author[1,2,3]{Yen-Huan Li}
\date{}

\affil[1]{Department of Computer Science and Information Engineering, National Taiwan University}
\affil[2]{Department of Mathematics, National Taiwan University}
\affil[3]{Center for Quantum Science and Engineering, National Taiwan University}

\maketitle

\begin{abstract}
Quantum state tomography (QST),  the task of estimating an unknown quantum state given measurement outcomes,  is essential to building reliable quantum computing devices. Whereas computing the maximum-likelihood (ML) estimate corresponds to solving a finite-sum convex optimization problem,  the objective function is not smooth nor Lipschitz,  so most existing convex optimization methods lack sample complexity guarantees; 
moreover,  both the sample size and dimension grow exponentially with the number of qubits in a QST experiment,  so a desired algorithm should be highly scalable with respect to the dimension and sample size,  just like stochastic gradient descent. 
In this paper,  we propose a stochastic first-order algorithm that computes an $\varepsilon$-approximate ML estimate in $O( ( D \log D ) / \varepsilon ^ 2 )$ iterations with $O( D^3 )$ per-iteration time complexity,  where $D$ denotes the dimension of the unknown quantum state and $\varepsilon$ denotes the optimization error. 
Our algorithm is an extension of Soft-Bayes to the quantum setup.
\end{abstract}

\section{Introduction} \label{sec_intro}

Quantum state tomography (QST),  the task of estimating an unknown quantum state given measurement outcomes,  is essential to building reliable quantum computing devices \cite{Paris2004}. 
The states of the quantum bits (qubits) prepared by an experimental apparatus are estimated, in order to check the correctness of the apparatus and, if needed, determine how to calibrate it.
Moreover, quantum process tomography, the task of estimating an unknown quantum channel, can also be cast as a QST problem \cite{Altepeter2003}.  
There are various approaches to QST,  such as trace regression \cite{Opatrny1997, Gross2010, Liu2011a, Flammia2012, Yang2020, Youssry2019},  maximum-likelihood (ML) estimation \cite{Hradil1997, Hradil2004, Bolduc2017},  Bayesian estimation \cite{Blume-Kohout2010, Blume-Kohout2010a},  and recently proposed deep learning-based methods \cite{Ahmed2020, Quek2021}. 
Among existing approaches,  the ML approach has been widely adopted for its relatively low estimation error in practice and asymptotic statistical guarantees in theory \cite{Hradil2004, Scholten2018}. 
\looseness=-1

Computing the ML estimator amounts to solving an optimization problem. 
Whereas the optimization problem is convex, standard convex optimization methods are not directly applicable. 
It is easily checked that the negative log-likelihood function in ML QST is neither Lipschitz nor smooth, violating standard assumptions in optimization literature \cite{Li2019a, Dvurechensky2020}. 
Hence, for example, even whether vanilla gradient descent converges for QST is unclear. 
\added{This is perhaps why $R \rho R$, a heuristic algorithm known to be empirically fast, was developed via an expectation maximization, instead of convex optimization, argument \cite{Lvovsky2004, MolinaTerriza2004}.
Unfortunately, $R \rho R$ does not always converge \cite{Rehacek2007}.}
The negative log-likelihood function is indeed self-concordant, so Newton's method is readily applicable \cite{Nesterov2004}. 
Nevertheless, the dimension of a quantum state grows exponentially with the number of qubits; 
the Hessian computations in Newton's method are computationally too expensive when the dimension is high. 
There are a few first-order \added{(i.e., gradient-based)} convex optimization algorithms provably converging for ML QST,  such as diluted \( R \rho R \) \cite{Rehacek2007, Goncalves2014}, SCOPT \cite{Tran-Dinh2015b}\footnote{A similar algorithm is proposed and studied in \cite{Gao2019}, but the bounded Hessian assumption therein renders the algorithm inapplicable to ML QST.}, NoLips \cite{Bauschke2017}, the Frank-Wolfe method \cite{Odor2016, Dvurechensky2020, Zhao2020, Carderera2021}, and entropic mirror descent with line search \cite{Li2019a}.
These are all batch methods. 
\added{As they require computing the full gradient in every iteration, }their per-iteration time complexities are at least linear in the sample size. 
To estimate a quantum state, \added{it has been proved that} the sample size must be exponential in the number of qubits \cite{ODonnell2016, Haah2017, Chen2022}. 
\looseness=-1

Regarding the high dimension and sample size issues in ML QST, it is desirable to, like how we handle the same issues in modern machine learning applications, develop a stochastic first-order optimization method for ML QST. 
\added{A stochastic first-order optimization method takes one or a few, instead of all, samples in each iteration and avoids computationally expensive Hessian computations.}
The stochastic quasi-Newton method for self-concordant minimization of \citet{Zhou2017} seems to apply. 
Nevertheless, its step size selection rule involves Hessian computations;
moreover, its analysis assumes a bounded Hessian, which does not hold in ML QST. 
The stochastic mirror-prox and stochastic primal-dual hybrid gradient methods were considered for problems very similar to ML QST \cite{Alacaoglu2021,Chambolle2018,He2019}. 
However, their analyses assume either a bounded dual domain or Lipschitzness; 
both are violated in ML QST. 


In this paper,  we propose a stochastic first-order algorithm for ML QST. 
We design the algorithm by an online learning argument. 
Consider an online convex optimization problem, where the loss function in each round corresponds to the negative log-likelihood function corresponding to one data point in ML QST.
Interestingly, this online convex optimization problem is exactly the quantum analogue of online portfolio selection, a celebrated online learning problem \cite{Cover1991, Cover1996}. 
Since the ML approach aims to minimize the empirical average of the negative log-likelihood,  once we ``quantumize'' any existing first-order online portfolio selection algorithm that is no-regret and apply an online-to-batch conversion \cite{Cesa-Bianchi2004,Cutkosky2019}, the resulting algorithm becomes a stochastic first-order algorithm for ML QST.  
We refer the reader to Section \ref{sec_background} for an introduction of relevant concepts. 

The algorithm we choose to ``quantumize'' is Soft-Bayes \cite{Orseau2017}. 
There are two reasons. 
First, the per-round time complexity of Soft-Bayes is linear in the ambient dimension, arguably the lowest one can expect; 
second, Soft-Bayes has a curious connection with expectation maximization \cite{Lvovsky2004, MolinaTerriza2004} (see Section \ref{sec_conclusion}).
We call the resulting algorithm Stochastic Q-Soft-Bayes. 
Stochastic Q-Soft-Bayes processes one randomly chosen data point in each iteration. 
Suppose the quantum state to be estimated is represented by a \( D \)-by-\( D \) density matrix. 
The per-iteration time complexity of Stochastic Q-Soft-Bayes is \( O ( D ^ 3 ) \), independent of the sample size. 
The expected optimization error of Stochastic Q-Soft-Bayes converges to zero at a \( O ( \sqrt{ ( 1 / T ) D \log D } ) \) rate,  where \( T \) denotes the number of iterations.
\looseness=-1

The main technical difficulty lies in figuring out an appropriate quantum extension of Soft-Bayes that coincides with Soft-Bayes when all matrices involved share the same eigenspace and allows for a regret analysis. 
This is challenging because for any given ``non-quantum'' expression, one can immediately find many candidates for its quantum extension, but only a few or one of them inherit the desired theoretical properties of their ``non-quantum'' counterpart; 
see, e.g., the discussion in \cite[Chapter 11]{Wilde2019} for extending information theoretic quantities to the quantum case. 
Similar to the quantum extension of exponentiated gradient update by \citet{Tsuda2005}, the quantum extension we find reveals the complicated mathematical structure of Soft-Bayes hidden in the ``non-quantum'' setup. 
\looseness=-1

Instead of empirically beating state of the arts, our aim is to give the first provably fast stochastic first-order algorithm for ML QST. 
Section \ref{sec_comparison} \added{shows that} Stochastic Q-Soft-Bayes is competitive in time complexity in comparison to existing batch algorithms. 
\added{Section \ref{sec_numerical} shows that Stochastic-Soft-Bayes is empirically even faster than $R \rho R$ in terms of the number of epochs.}
Unfortunately, Section \added{\ref{app_numerical}}
shows that \added{in terms of the elapsed time, }
Q-Soft-Bayes \added{may not be}
satisfactory to practitioners. 
We discuss the possibility of developing faster stochastic first-order methods in Section \ref{sec_conclusion}. 
\looseness=-1

\subsection{Related work}

A textbook approach to quantum state tomography is to approximate the problem as a trace regression problem \cite{Nielsen2010} and compute the corresponding least-squares estimate or directly minimize the expected square loss,  sometimes with regularization \cite{Opatrny1997, Gross2010, Flammia2012, Youssry2019, Yang2020}. 
Since minimizing the square loss is arguably the most standard problem in optimization and machine learning, many existing algorithms apply. 
\citet{Youssry2019} proved the convergence of stochastic entropic mirror descent. 
\citet{Yang2020} showed that several standard online learning algorithms are no-regret for the corresponding online trace regression problem. 
Notice that both papers do not consider the ML formulation. 

Quantum state tomography schemes optimal or nearly optimal in sample complexity \added{are known}
\cite{ODonnell2016,Haah2017,Kueng2017,Guta2020}. 
\added{The optimal schemes} require entangled measurements, challenging to implement \cite{ODonnell2016,Haah2017}. 
\added{If only incoherent measurements (as in the ML QST scheme considered in this paper) are allowed, the scheme by Kueng et al. \cite{Kueng2017} is optimal \cite{Chen2022};
nevertheless, the scheme is still challenging to implement \cite[p. 97]{Kueng2017}. }
The scheme proposed by Gu\c{t}\v{a} et al. \cite{Guta2020} \added{is nearly optimal, but}
the numerical result in \cite[Figure 1]{Guta2020} shows that the ML approach achieves a smaller estimation error empirically. 
\looseness=-1

A problem closely related to quantum state tomography is shadow tomography,  in which one is not interested in recovering the quantum state but estimating the probability distributions of its measurement outcomes \cite{Aaronson2020}. 
Aaronson et al. showed that shadow tomography can be done in an online fashion,  via follow the regularized leader with the von Neumann entropy \cite{Aaronson2018}. 
We emphasize that shadow tomography is fundamentally different from quantum state tomography.  
Indeed,  Aaronson showed shadow tomography is strictly easier than state tomography,  in the sense that the former requires much less samples than the latter \cite{Aaronson2020}.
Another closely related problem is the quantum version of individual sequence prediction considered by \citet{Koolen2011}. 
The loss function studied in \cite{Koolen2011} is the trace-log loss,  instead of the log-trace loss we consider,  as discussed in Section 4 of their paper. 
\looseness=-1

Our algorithm is developed via ``quantumizing'' an online portfolio selection algorithm. 
Online portfolio selection is a classic online learning problem. 
It is known that the optimal regret of online portfolio selection is \( O ( D \log T ) \),  where \( D \) denotes the ambient dimension and \( T \) denotes the number of rounds, and is achieved by Universal Portfolio Selection (UPS) \cite{Cover1991, Cover1996}. 
However, UPS is computationally too expensive to be practical \cite{Kalai2002}. 
There are several algorithms that try to balance between the regret and computational complexity,  but none of them is optimal in both aspects \cite{Luo2018, vanErven2020}. 
Soft-Bayes strikes a balance with a $O ( D )$ per-round time complexity and $O ( \sqrt{ T D \log D } )$ regret. 
\looseness=-1

Recently, \citet{Zimmert2022} ``quantumized'' another online portfolio selection algorithm, called BISONS, to solve the game of online quantum state tomography described in Section \ref{sec_game}.
By an online-to-batch conversion, their algorithm yields a stochastic algorithm for ML QST. 
The resulting algorithm achieves a better iteration complexity than Stochastic Q-Soft-Bayes; 
nevertheless, each iteration of it requires solving a self-concordant convex program by,  e.g., Newton's method,  resulting in a high time complexity incomparable to that of Stochastic Q-Soft-Bayes.    
In the words of \citet{Zimmert2022}, both their and our algorithms are on the state-of-the-art efficiency-regret frontier. 

\subsection{Notations}
We write \( \mathbb{R}_+ \) for the set of non-negative real numbers and \( \mathbb{R}_{++} \) the set of strictly positive real numbers. 
Let \( J \in \mathbb{N} \). 
We write \( [ J ] \) for the set \( \set{ 1,  \ldots,  J } \). 
Let \( M \) be a matrix. 
We write \( M^{\mathrm{H}} \) for its Hermitian (conjugate transpose) and \( \tr ( M ) \) for its trace. 
Let \( H \) be a Hermitian matrx; 
we write its spectral decomposition as \( H = \sum_d \lambda_d P_d \),  where \( \lambda_d \) are the eigenvalues and \( P_d \) are projections onto the associated eigenspaces. 
Let \( f \) be a real-valued function whose domain contains \( \set{ \lambda_d } \). 
Then,  \( f ( H ) \) is defined as the matrix \( \sum_d f ( \lambda_d ) P_d \). 
Let \( A \) and \( B \) be two matrices. 
We write \( A \geq B \) if and only if \( A - B \) is positive semi-definite. 
Let \( \mathcal{E} \) be an event and \( \xi \) be a random variable following a probability distribution \( P \). 
We write \( \mathsf{P} ( \mathcal{E} ) \) for the probability of the event and \( \mathsf{E}_P \left[ \xi \right] \) for the expectation of \( \xi \). 
We sometimes omit the subscript \( P \) and write \( \mathsf{E} \left[ \xi \right] \) when there is no ambiguity. 

\section{Preliminaries} 
\label{sec_background}


\subsection{Maximum-Likelihood Quantum State Tomography} \label{sec_mlqst}

In the mathematical formulation of quantum mechanics, a quantum state corresponds to a \emph{density matrix}, a Hermitian positive semi-definite complex matrix of unit trace. 
Let the dimension of the density matrix be \( D \in \mathbb{N} \). 
If there are \( q \) qubits,  then \( D = 2 ^ q \). 
We denote by \( \mathcal{D} \) the set of density matrices in \( \mathbb{C}^{D \times D} \),  i.e.,  
\[
\mathcal{D} \coloneqq \Set{ \rho | \rho \in \mathbb{C}^{D \times D},  \rho = \rho^{\mathrm{H}},  \rho \geq 0,  \tr \rho = 1 } . 
\]
A measurement setup corresponds to a \emph{positive operator-valued measure (POVM)}, a set of Hermitian positive semi-definite complex matrices summing to the identity. 
Let \( \rho \in \mathcal{D} \) and \( \set{ M_1,  \ldots,  M_J } \subset \mathbb{C}^{D \times D} \) be a POVM. 
The measurement outcome is a random variable \( \eta \) taking values in \( [ J ] \) such that 
\[
\mathsf{P} \left( \eta = j \right) = \tr ( M_j \rho ) ,  \quad \forall j \in [ J ] . 
\]

The ML estimation approach seeks the quantum state that maximizes the probability of observing the measured data. 
Let \( \rho^\natural \in \mathcal{D} \) be the density matrix to be estimated. 
In a standard QST experiment,  we construct \( N \) independent copies of \( \rho^\natural \) and measure the copies independently with possibly different POVMs. 
It is easily checked that the ML estimator is of the form 
\[
\hat{\rho} \in \argmax_{\rho \in \mathcal{D}} \prod_{n = 1}^N \tr ( A_n \rho ) ,  
\]
where each \( A_n \) is an element of the POVM for the \( n \)-th measurement. 
We call \( \set{ A_1,  \ldots,  A_N } \) the \emph{data-set}. 
We equivalently write the ML estimator as 
\begin{align}
& \hat{\rho} \in \argmin_{\rho \in \mathcal{D}} f ( \rho ) ,  \label{eq_problem} \\
& f ( \rho ) \coloneqq \frac{1}{N} \sum_{n = 1}^N \left( - \log \tr ( A_n \rho ) \right) . \label{eq_f}
\end{align}
Obviously,  \eqref{eq_problem} is a convex optimization problem. 
If the matrix \( A_{n'} \) is not full-rank for some \( n' \in [ N ] \) (as in the cases with the Pauli measurement \cite{Liu2011} and Pauli basis measurement \cite{Riofrio2017, Steffens2017}),  then \( \tr ( A_{n'} \rho ) \) can be arbitrarily close to zero on \( \mathcal{D} \) and hence the \( k \)-th-order derivative of the objective function \( f \) is unbounded for all \( k \in \mathbb{N} \). 

Let \( A \) be a random matrix following the empirical distribution \( \hat{P}_N \) on the data-set \( \set{A_1,  \ldots,  A_N} \). 
If the matrices \( A_n \) are all different,  then \( \hat{P}_N \) is simply the uniform distribution on the data-set \( \set{ A_1,  \ldots,  A_N } \). 
Then,  we can write the objective function in \eqref{eq_problem} as an expectation
\begin{equation}
f ( \rho ) = \mathsf{E}_{\hat{P}_N} \left[ - \log \tr ( A \rho ) \right] . \label{eq_f_expectation}
\end{equation}
This observation connects ML QST with the problem of computing the log-optimal portfolio. 

\subsection{Log-optimal Portfolio}

Interestingly,  the optimization problem \eqref{eq_problem} is exactly a quantum extension of the problem of computing the log-optimal portfolio (aka the Kelly criterion),  an asymptotically optimal strategy for long-term investment \cite{Algoet1988,Breiman1975,Kelly1956}. 
Consider multi-round investment in a market. 
Suppose there are \( D \) investment alternatives. 
For the \( t \)-th round,  we list the return rates of the investment alternatives in that round as a random vector \( a_{t} \in \mathbb{R}_+^D \). 
Before each round starts,  an investor needs to determine the portfolio for the round given the past return rates. 
Denote by \( P_{t + 1} \) the probability distribution of \( a_{t + 1} \) conditional on the history \( ( a_1,  \ldots,  a_t ) \). 
The log-optimal portfolio $w^\star_{t + 1}$ for the $(t + 1)$-th round is given by the stochastic optimization problem: 
\begin{align}
& w^\star_{t + 1} \in \argmin_{w \in \Delta} \varphi ( w ) ,  \label{eq_kelly} \\
& \varphi ( w ) \coloneqq \mathsf{E}_{P_{t + 1}} \left[ - \log \braket{ a_{t + 1},  w } \right] ,  \label{eq_phi}
\end{align}
where \( \Delta \) denotes the probability simplex in \( \mathbb{R}^D \),  the set of entry-wise non-negative vectors whose entries sum to one. 
Then,  the investor distributes the wealth to the investment alternatives following the ratios specified by \( w^\star_{t + 1} \). 

We now discuss the correspondence between ML QST and log-optimal portfolio. 
The set \( \mathcal{D} \) is indeed a quantum extension of the probability simplex \( \Delta \),  in the sense that a Hertimian matrix is a density matrix if and only if its vector of eigenvalues lies in the probability simplex. 
The objective functions in \eqref{eq_problem} and \eqref{eq_kelly} are both expectations of the logarithm of linear functions. 
Indeed,  it is easily checked that if the matrices involved in \eqref{eq_problem} share the same eigenbasis, then the non-commutativity issue in the quantum setup vanishes and \eqref{eq_problem} coincides with \eqref{eq_kelly}. 
Though the correspondence is obvious given the two problem formulations,  it seems that this correspondence has not been discussed in the literature. 

\subsection{Online Portfolio Selection}

Online portfolio selection may be viewed as a {probability-free} version of log-optimal portfolio \cite{Cover1991}. 
Online portfolio selection is a multi-round game between two players,  say \textsc{Investor} and \textsc{Market}. 
Suppose the game consists of \( T \) rounds. 
In the \( t \)-th round of the game,  first,  \textsc{Investor} announces a portfolio \( w_t \in \Delta \); 
then,  \textsc{Market} announces the return rates of all investment alternatives for the \( t \)-th round in a vector \( a_t \in \mathbb{R}_{+}^D \); 
finally,  \textsc{Investor} suffers a loss of value \( - \log \braket{ a_t,  w_t } \). 
The goal of \textsc{Investor} is to achieve a low \emph{regret} against all possible strategies of \textsc{Market}. 
The regret is defined as 
\looseness=-1
\[
R_T \coloneqq \sup \sum_{t = 1}^T \left( - \log \braket{ a_t,  w_t } \right) - \min_{w \in \Delta} \sum_{t = 1}^T \left( - \log \braket{ a_t,  w } \right) ,  
\]
where the supremum is over all possible strategies of \textsc{Market} to determine \( ( a_t )_{1 \leq t \leq T} \). 
We say an algorithm for \textsc{Investor} to determine the portfolios is \emph{no-regret} if it achieves \( R_T = o ( T ) \). 

If we can sample from the conditional probability distribution \( P_{t + 1} \) specified in the previous sub-section,  then we can transform a no-regret online portfolio selection algorithm to an algorithm that approximately computes the log-optimal portfolio. 
The following is an immediate consequence of the online-to-batch conversion \cite{Cesa-Bianchi2004, Orabona2019}. 

\begin{proposition} \label{prop_otb_ops}
Suppose in the online portfolio selection game,  the vectors \( a_t \) are all independent and identically distributed (i.i.d.) random vectors following the 
probability distribution \( P_{t + 1} \) in the previoius sub-section. 
Let \( ( w_t )_{t \in \mathbb{N}} \) be the sequence of iterates generated by an algorithm for \textsc{Investor} of regret \( R_T \). 
Then,  for any \( T \in \mathbb{N} \),  
\[
\mathsf{E} \left[ \varphi ( \overline{w}_T ) - \min_{w \in \Delta} \varphi ( w ) \right] \leq \frac{ R_T }{ T } ,  \quad \overline{w}_T \coloneqq \frac{ w_1 + \cdots + w_T }{ T } . 
\] 
Recall that \( \varphi \) is the conditional expectation of the log-linear loss in \eqref{eq_phi}. 
\end{proposition}

If \textsc{Investor} adopts a no-regret algorithm,  then the expected optimization error vanishes as \( T \to \infty \). 

\subsection{Soft-Bayes} \label{sec_softBayes}

There are various existing algorithms for online portfolio selection. 
Among these algorithms,  we are particularly interested in Soft-Bayes \cite{Orseau2017}. 
The per-iteration time complexity of Soft-Bayes is linear in \( D \), arguably the lowest one can expect. 
This is a desirable feature for ML QST, as the dimension of the density matrix grows exponentially with the number of qubits. 

The iteration rule of Soft-Bayes is as follows. 
\begin{itemize}
\item Initialize at \( w_1 = ( 1 / D,  \ldots,  1 / D ) \in \Delta \) (the uniform distribution). 
\item For each \( t \in \mathbb{N} \),  compute 
\begin{equation}
w_{t + 1} = ( 1 - \eta ) w_t + \eta \frac{ a_t \circ w_t }{ \braket{ a_t,  w_t } } ,  \quad \forall t \in \mathbb{N} ,  \label{eq_soft_bayes}
\end{equation}
for some properly chosen \emph{learning rate} \( \eta \in \interval{0}{1} \),  where \( \circ \) denotes the entry-wise product. 
\end{itemize}
Soft-Bayes has the following regret guarantee. 


\begin{theorem}[\cite{Orseau2017}] \label{thm_soft_bayes}
After \( T \) rounds in online portfolio selection,  the regret of Soft-Bayes with 
\begin{equation}
\eta = \frac{ \sqrt{ D T } }{ \sqrt{ D T } + \sqrt{ \log D } } \label{eq_learning_rate}
\end{equation}
is at most \( 2 \sqrt{ T D \log D } + \log D \). 
\end{theorem}

\section{Online Maximum-Likelihood Quantum State Tomography by Q-Soft-Bayes}

Following the discussion in Section \ref{sec_intro}, we first propose a \emph{game of online quantum state tomography} as a quantum extension of online portfolio selection. 
Then,  we ``quantumize'' Soft-Bayes to derive a no-regret algorithm for the game and analyse its regret. 
Finally,  we adopt the online-to-batch conversion and bound the expected optimization error of the resulting algorithm. 

\subsection{Game of Online Quantum State Tomography} \label{sec_game}

We propose the following game of online quantum state tomography as a quantum extension of online portfolio selection. 
Online quantum state tomography is a multi-round game between two players,  say \textsc{Physicist} and \textsc{Environment}. 
Suppose there are in total \( T \) rounds. 
In the \( t \)-th round,  first,  \textsc{Physicist} announces a density matrix \( \rho_t \in \mathcal{D} \); 
then,  \textsc{Environment} announces a Hermitian positive semi-definite matrix \( A_t \geq 0 \); 
finally,  \textsc{Physicist} suffers for a loss of value \( - \log \tr ( A_t \rho_t ) \). 
The regret in this game is given by
\begin{align*}
R_T \coloneqq \sup \sum_{t = 1}^T \left( - \log \tr ( A_t \rho_t ) \right) - \min_{\gamma \in \mathcal{D}} \sum_{t = 1}^T \left( - \log \tr ( A_t \rho ) \right) ,  
\end{align*}
where the supremum is over all possible strategies of \textsc{Physicist} to generate the sequence \( ( A_t )_{1 \leq t \leq T} \). 

The connection with online portfolio selection is obvious and similar to that between ML QST and the log-optimal portfolio: 
The vector of eigenvalues of a density matrix lies in the probability simplex \( \Delta \); 
the Hermitian matrices \( A_t \) and the positive semi-definiteness condition correspond to the vectors \( a_t \) in online portfolio selection and their non-negativity condition,  respectively; 
the losses in the two games are both logarithms of linear functions; 
the regrets in the two games are defined exactly in the same manner. 
When all the matrices involved in the game of online quantum state tomography share the same eigenbasis,  we recover the game of online portfolio selection. 

\subsection{Q-Soft-Bayes and Stochastic Q-Soft-Bayes}

We propose the following Q-Soft-Bayes algorithm as a quantum extension of Soft-Bayes. 
\begin{itemize}
\item Initialize at \( \rho_1 = W_1 = I / D \). 
\item For each \( t \in \mathbb{N} \),  compute 
\begin{equation}
\begin{aligned}
& G_t = ( 1 - \eta ) I + \eta \frac{A_t}{\tr ( A_t \rho_t )} ,  \\
& W_{t + 1} = \exp \left( \log \left( W_t \right) + \log \left( G_T \right) \right) ,  \\
& \rho_{t + 1} = \frac{W_{t + 1}}{\tr ( W_{t + 1} )} ,  
\end{aligned} \label{eq_qsb}
\end{equation}
for some properly chosen learning rate \( \eta \in \interval{0}{1} \). 
\end{itemize}

\begin{remark}
Recently, we learned that Q-Soft-Bayes may be interpreted using the \emph{commutative matrix product} by \citet{Warmuth2010}. 
It is currently unclear to us whether this interpretation provides any insight. 
\end{remark}

If we were able to cancel the exponential and logarithms in Q-Soft-Bayes, then we recover Soft-Bayes; 
however, due to the non-commutativity issue, such cancellation is illegal in general. 
In comparison to Soft-Bayes, Q-Soft-Bayes has an additional normalization step to ensure its outputs are of unit trace. 
We prove the following in Section \ref{sec_small_trace}. 

\begin{proposition} \label{prop_small_trace}
It holds that \( \tr ( W_t ) \leq 1 \) for all \( t \). 
\end{proposition}

Numerical experiments show that the equality does not always hold, so the normalization step is necessary. 
Recall that Soft-Bayes does not need the normalization step (see Section \ref{sec_softBayes}). 

In Appendix \ref{app_qsb},  we prove the following regret bound for Q-Soft-Bayes,  showing that it inherits the regret bound of Soft-Bayes. 

\begin{theorem} \label{thm_qsb}
The regret of Q-Soft-Bayes with the learning rate \( \eta \) given in \eqref{eq_learning_rate} is at most \( 2 \sqrt{ T D \log D } + \log D \). 
\end{theorem}

\begin{remark}
One might wonder why $D$ is not replaced by $D^2$ in the quantum case. 
This is because in our extension, the analogue of a $D$-dimensional vector in the quantum case is the $D$-dimensional vector of eigenvalues of a $D$-by-$D$ Hermitian matrix,  instead of the $D$-dimensional vector obtained by vectorizing a $\sqrt{D}$-by-$\sqrt{D}$ matrix. 
Similar coincidence of regret bounds can be observed in, for example, the matrix version of exponentiated gradient update \cite{Tsuda2005, Arora2012} and the quantum individual sequence prediction algorithms of \citet{Koolen2011}. 
\end{remark}

The standard online-to-batch conversion argument can also be applied to solving ML QST by Q-Soft-Bayes. 
Recall for ML QST,  our aim is to solve the stochastic optimization problem
\[
\hat{\rho} \in \argmin_{\rho \in \mathcal{D}} \mathsf{E}_{\hat{P}_N} \left[ - \log \tr ( A,  \rho ) \right] ,  
\]
where \( A \) is a random matrix following the empirical probability distribution \( \hat{P}_N \) on the data-set \( \set{ A_1,  \ldots,  A_N } \) (see Section \ref{sec_mlqst}). 
We propose the following stochastic optimization algorithm,  which we call Stochastic Q-Soft-Bayes,  to solve ML QST. 
\begin{itemize}
\item Initialize Q-Soft-Bayes with \( \rho_1 = W_1 = I / D \).
\item In the \( t \)-th iteration of Stochastic Q-Soft-Bayes,  do the following. 
	\begin{enumerate}
	\item Output the \( t \)-th output \( \rho_t \) of Q-Soft-Bayes. 
	\item Sample a random matrix \( B_{t} \in \set{ A_1,  \ldots,  A_N } \) following the empirical probability distribution \( \hat{P}_N \) on the data-set,  independent of the past. 
	\item Let \textsc{Environment} in the online QST game announce the matrix \( B_{t} \). 
	\end{enumerate}
\end{itemize}

Similarly as for Proposition \ref{prop_otb_ops},  the standard online-to-batch conversion argument provides the following convergence guarantee of Stochastic-Q-Soft-Bayes. 

\begin{proposition} \label{prop_otb_qst}
Let \( ( \rho_t )_{t \in \mathbb{N}} \) be the sequence of iterates generated by Stochastic Soft-Bayes. 
Then,  for any \( T \in \mathbb{N} \),  it holds that 
\[
\mathsf{E} \left[ f ( \overline{\rho}_T ) - \min_{\rho \in \mathcal{D}} f ( \rho ) \right] \leq 2 \sqrt{ \frac{ D \log D }{ T } } + \frac{\log D}{T} ,  
\]
where \( \overline{\rho}_T \coloneqq ( \rho_1 + \cdots + \rho_T ) / T \) and the expectation is with respect to the randomness in \( B_t \) of Stochastic Soft-Bayes. 
Recall \( f \) is the objective function in ML QST as defined in \eqref{eq_f} or \eqref{eq_f_expectation} (the two definitions are equivalent). 
\end{proposition}

Therefore,  Stochastic-Q-Soft-Bayes outputs an approximate ML estimator of expected optimization error smaller than \( \varepsilon \) in \( O ( ( D \log D ) / \varepsilon ^ 2 ) \) iterations. 
Each iteration of Stochastic-Q-Soft-Bayes requires computing a matrix exponential and two matrix logarithms. 
The overall time complexity is hence \( O ( ( D ^ 4 \log D ) / \varepsilon ^ 2 ) \). 
\added{One may adopt anytime online-to-batch \cite{Cutkosky2019}, which seems to empirically yield faster convergence. 
According to \cite{Cutkosky2019}, the optimization error guarantee remains the same; the only difference is that $\nabla f$ are evaluated at $\overline{\rho}_t$ instead of $\rho_t$ when implementing Soft-Bayes, so the overall time complexity also remains the same.}

\added{One may be interested in the distance to the minimizer. 
It is easily checked that the function $f$ is self-concordant. 
If $\nabla^2 f$ is positive definite at the minimizer, a standard condition for well-posed estimators, then the function $f$ is locally strongly convex around the minimizer \cite[Theorem 4.1.6]{Nesterov2004}.
Therefore, the distance to the minimizer, measured in terms of the Frobenius norm, is asymptotically of the order of the square root of the optimization error.}

\subsection{Theoretical Comparison with Existing Batch Algorithms} \label{sec_comparison}

Let us compare the time complexities of Stochastic Q-Soft-Bayes and existing algorithms discussed in Section \ref{sec_intro}. 
The iteration complexities of existing algorithms are mostly unknown or vague in their dependence on the problem parameters. 
Diluted $R \rho R$ and entropic mirror descent with line search do not have non-asymptotic analysis results \cite{Rehacek2007, Goncalves2014, Li2019a}; 
SCOPT only has a local linear rate guarantee \cite{Tran-Dinh2015b}; 
Adaptive Frank-Wolfe and Monotonous Frank-Wolfe have \( O ( \varepsilon^{-1} ) \) iteration complexities with unclear dependence on the dimension and sample size,  as their error bounds involve local smoothness parameters that are hard to evaluate \cite{Carderera2021, Dvurechensky2020}.
A finer analysis of Adaptive Frank-Wolfe by \citet{Zhao2020} shows that its iteration complexity is \( O ( \varepsilon^{-1} N ) \) and hence its time complexity is \( O \left( \varepsilon^{-1} ( N^2 D^2 + N \tau ) \right) \),  where the symbol \( \tau \) denotes the time of computing the local norm defined by the Hessian,  for which we do not know an efficient implementation. 
In comparison,  the complexities of Stochastic Q-Soft-Bayes is very clear: \( O ( \varepsilon^{-2} D \log D ) \) iteration complexity and hence \( O ( \varepsilon^{-2} D^4 \log D ) \) time complexity. 
The time complexity of Stochastic Q-Soft-Bayes becomes competitive with Adaptive Frank-Wolfe if $N \gg D \sqrt{ ( 1 / \varepsilon ) \log D }$,  \emph{ignoring the time of computing the local norms}. 
\added{Recently, it is proved that any QST scheme with non-coherent measurement, e.g., ML QST we consider in this paper, requires $N = \Omega ( D ^ 3 / \delta ^ 2 )$ to achieve an estimation error smaller than $\delta$ in the trace distance \cite{Chen2022}. }
The algorithm by \citet{Zimmert2022} has a \( \tilde{O} ( D ^ 3 / \varepsilon ) \) iteration complexity and \( O ( D^6 ) \) per-iteration time complexity ignoring the dependence on other parameters, due to the use of Newton's method to compute the iterates; 
the overall time complexity has a much higher dependence on the dimension than Adaptive Frank-Wolfe and Stochastic Q-Soft-Bayes. 
We conclude that the time complexity of Stochastic Q-Soft-Bayes is competitive \added{compared to}
existing algorithms. 
\looseness=-1

\section{Numerical Results} \label{sec_numerical}

\begin{figure}[!t]

\begin{minipage}{0.48\textwidth}
\includegraphics[scale=.4]{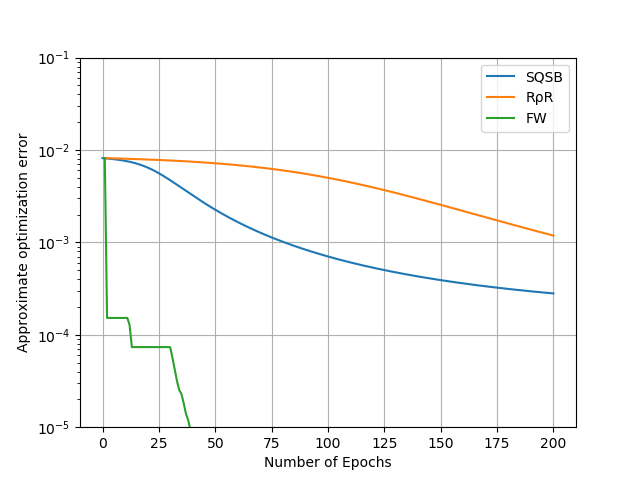}
\caption{Approximate optimization errors in function value of Stochastic Q-Soft-Bayes (SQSB),  \( R \rho R \) \deleted{(RrhoR)},  and Monotonous Frank-Wolfe (FW). 
}
\label{fig_error}
\end{minipage} \hfill 
\begin{minipage}{.48\textwidth}
\includegraphics[scale=.4]{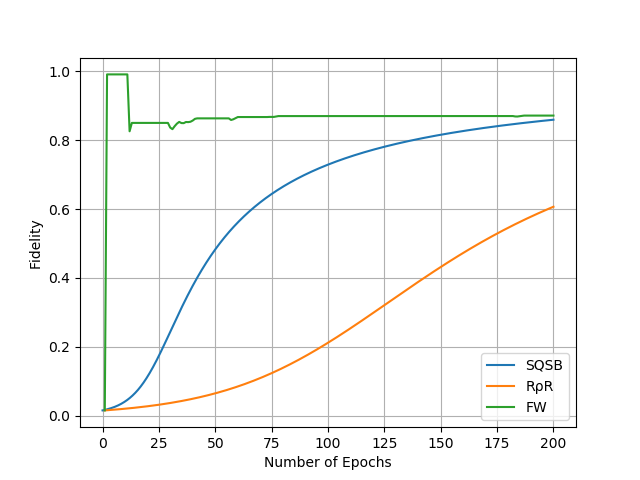}
\caption{Fidelity values of the iterates and the W state achieved by Stochastic Q-Soft-Bayes (SQSB),  \( R \rho R \) \deleted{(RrhoR)},  and Monotonous Frank-Wolfe (FW).}
\label{fig_fidelity}
\end{minipage}

\end{figure}

As discussed above, \added{Stochastic Q-Soft-Bayes}
is competitive in \added{theory.}
\added{We now examine its empirical performance with anytime online-to-batch.}
We compare its empirical speed with two batch methods,  the \( R \rho R \) method \cite{Lvovsky2004, MolinaTerriza2004} and Monotonous Frank-Wolfe \cite{Carderera2021},  on a synthetic data-set in Figure \ref{fig_error} and Figure \ref{fig_fidelity}. 
We have mentioned several batch methods applicable for ML QST in Section 1. 
Among them,  we choose \( R \rho R \) for comparison as it is representative in physics literature and empirically fast, though \emph{it does not always converge}. 
We choose monotonous Frank-Wolfe for comparison as it avoids computationally expensive Hessian computations in step size selection\added{.}
\added{Recall that Monotonous Frank-Wolfe} converges at a \( O ( 1 / t ) \) rate as other Frank-Wolfe methods for self-concordant minimization do \cite{Zhao2020, Dvurechensky2020, Odor2016}\added{, but its complexity guarantee lacks a clear characterization of the dependence on the dimension and sample size}. 

The synthetic data-set is generated basically following the set-up in \cite{Haffner2005}.
The number of qubits \( q \) equals \added{6}. 
The dimension \( D \) then equals {\( 2 ^ q = 64 \)}. 
The unknown quantum state to be measured is the W state. 
We randomly generate {\( N = 4 ^ q \times 100 = 409600 \)} Pauli observables as in, e.g., \cite{Flammia2012,Gross2010,Liu2011},  each of which corresponds to a POVM of two rank-\( ( D / 2 ) \) elements. 
As there are in total \( 4 ^ q \) different Pauli observables (and hence POVMs), each observable is used about 100 times. 
Then,  we sample the \( N \) measurement outcomes and formulate the ML estimator following Section 2.1. 


The performance measures we consider are optimization errors (in objective function) and fidelity values. 
\added{To estimate the optimization error}, we run each algorithm for 200 epochs and use the smallest function value found by the algorithms
as an approximate optimal value. 
The approximate optimization error of an iterate is defined as the difference between the objective function value at the iterate and the approximate optimal value. 
Fidelity is a notion commonly used by physicists to measure how close two quantum states are to each other. 
For any two density matrices $\rho$ and $\sigma$,   the fidelity is given by \( F ( \rho,  \sigma ) \coloneqq \left( \tr \sqrt{ \sqrt{ \rho } \sigma \sqrt{\rho} } \right) ^ 2 \),   which takes values in \( \interval{0}{1} \). 
The fidelity of two quantum states equals 1,  if the two states are exactly the same. 
We plot the optimization errors and fidelity values versus the number of epochs. 
An epoch corresponds to one pass of the whole data-set.
One iteration of Stochastic Q-Soft-Bayes corresponds to \( 1 / N \) epoch. 
One iteration of \( R \rho R \) and Monotonous Frank-Wolfe corresponds to 1 epoch as both algorithms are batch. 

\added{Obviously, Stochastic Q-Soft-Bayes converges faster than $R \rho R$ in both optimization error and fidelity.
Where as Monotonous Frank-Wolfe is the fastest in both figures, this can}
be explained by the fact that Frank-Wolfe tends to generate approximately low-rank iterates. 
The W state corresponds to a rank-1 density matrix, so the ML estimate should be approximately low-rank, matching the structure favoured by Frank-Wolfe.
\added{We conclude that the convergence speed of Stochastic Q-Soft-Bayes is competitive in theory (Section \ref{sec_comparison}) and comparable to fast yet theoretically non-rigorous algorithms in practice.}
\looseness=-1

\added{A comparison}
in terms of the elapsed time \added{is provided} in Appendix \ref{app_numerical}. 
\added{The results show there is a large room for improvement to compete with $R \rho R$ and Monotonous Frank-Wolfe in the elapsed time.}
The \added{source} codes are provided in the supplementary material. 

\section{Discussions} \label{sec_conclusion}

\subsection{Can We Find a Faster Stochastic First-Order Algorithms for ML QST?} \label{sec_faster}

Our approach to constructing a stochastic first-order algorithm for ML QST conceptually applies to any no-regret online portfolio selection algorithm. 
In this paper, we focus on Soft-Bayes. 
Other existing online portfolio selection algorithms have much higher per-iteration time complexities, in terms of the dependence on the ambient dimension and sample size. 
If we adopt any other existing online portfolio selection algorithm and ``quantumize'' it to obtain a stochastic algorithm for ML QST, then the resulting algorithm will scale poorly with the number of qubits. 
Developing an online portfolio selection algorithm that enjoys both a low regret and low time complexity is still 
open 
\cite{vanErven2020,Zimmert2022}.

It is still possible to develop another quantum extension of Soft-Bayes that enjoys a lower per-iteration time complexity. 
The per-iteration time complexity issue may be mitigated if we consider other quantum extensions of Soft-Bayes. 
For example,  if we na\"{\i}vely replace (8) by $W_{t + 1} = \left( G_t W_t + W_t G_t \right) / 2$,  the resulting algorithm still coincides with Soft-Bayes when all matrices share the same eigenbasis,  whereas the per-iteration time complexity is reduced to \added{$O ( D^{\omega} )$ for some $\omega < 2.373$ \cite{Alman2021}.} 
Unfortunately,  we 
\added{cannot work out a} non-asymptotic analysis for any other possible quantum extension of Soft-Bayes we can think of. 

The discussion above assumes that we adopt the online-to-batch argument as in this paper. 
Another way, which we think perhaps more plausible, is to directly consider the stochastic optimization formulation and develop a stochastic optimization algorithm for ML QST. 

\subsection{Connection with Expectation Maximization}

Finally,  let us discuss an interesting connection between Q-Soft-Bayes and expectation maximization (EM). 
The \( R \rho R \) algorithm, according to \cite{Lvovsky2004, MolinaTerriza2004}, was inspired by the expectation maximization (EM) method for solving optimization problems of the form \eqref{eq_kelly}. 
Given a full-rank initial iterate \( \rho_1 \in \mathcal{D} \),  \( R \rho R \) iterates as 
\[
\rho_{t + 1} = \frac{ R_t \rho_t R_t }{ \tr ( R_t \rho_t R_t ) } ,  \quad R_t \coloneqq - \nabla f ( \rho_t ) ,  \quad \forall t \in \mathbb{N} ,  
\]
where \( f \) is defined in \eqref{eq_f}. 
In comparison, given an entry-wise positive vector \( w_1 \in \Delta \),  EM for \eqref{eq_kelly} iterates as 
\[
w_{t + 1} = w_t \circ ( - \nabla \varphi ( w_t ) ) ,  \quad \forall t \in \mathbb{N} . 
\]
It is interesting to notice that even when all matrices involved share the same eigenbasis,  \( R \rho R \) is not equivalent to EM. 
Indeed,  EM is proved to asymptotically converge to the optimum \cite{Cover1984, Csiszar1984},  whereas \( R \rho R \) oscillates on a carefully designed data-set \cite{Rehacek2007}. 
This suggests that \( R \rho R \) is perhaps not a ``natural'' quantum extension of EM. 
Later,  there were variations of \( R \rho R \) that solve the convergence issue by line search \cite{Rehacek2007, Goncalves2014},  but these variations still do not recover EM. 

Notice that the formulation of Soft-Bayes \eqref{eq_soft_bayes} is the convex combination of the previous iterate and \emph{the output of EM}. 
Therefore,  Soft-Bayes,  after the online-to-batch conversion,  can be interpreted as a relaxed stochastic EM method for computing the log-optimal portfolio. 
As Q-Soft-Bayes becomes Soft-Bayes when all matrices involved share the same eigenbasis,  we may claim that Stochastic Q-Soft-Bayes is also a relaxed stochastic EM method,  though its derivation does not have any obvious relation with the standard derivation of EM \cite{Dempster1977}. 

\section*{Acknowledgements}

This work is supported by the Young Scholar Fellowship (Einstein Program) of the Ministry of Science and Technology of Taiwan under grant numbers MOST 109-2636-E-002-025,  MOST 110-2636-E-002-012, and MOST 111-2636-E-002-019 and by the research project “Pioneering Research in Forefront Quantum Computing, Learning and Engineering” of National Taiwan University under grant number NTU-CC-111L894606.


\bibliography{list}

\begin{thebibliography}{67}
\providecommand{\natexlab}[1]{#1}
\providecommand{\url}[1]{\texttt{#1}}
\expandafter\ifx\csname urlstyle\endcsname\relax
  \providecommand{\doi}[1]{doi: #1}\else
  \providecommand{\doi}{doi: \begingroup \urlstyle{rm}\Url}\fi

\bibitem[Aaronson(2020)]{Aaronson2020}
S.~Aaronson.
\newblock Shadow tomography of quantum states.
\newblock \emph{SIAM J. Comput.}, 49\penalty0 (5):\penalty0
  STOC18--368--STOC18--394, 2020.

\bibitem[Aaronson et~al.(2018)Aaronson, Chen, Hazan, Kale, and
  Nayak]{Aaronson2018}
S.~Aaronson, X.~Chen, E.~Hazan, S.~Kale, and A.~Nayak.
\newblock Online learning of quantum states.
\newblock In \emph{Adv. Neural Information Processing Systems 31}, 2018.

\bibitem[Ahmed et~al.(2020)Ahmed, Mu\~{n}oz, Nori, and Kockum]{Ahmed2020}
S.~Ahmed, C.~S. Mu\~{n}oz, F.~Nori, and A.~F. Kockum.
\newblock Quantum state tomography with conditional generative adversarial
  networks, 2020.
\newblock arXiv:2008.03240.

\bibitem[Alacaoglu(2021)]{Alacaoglu2021}
A.~Alacaoglu.
\newblock \emph{Adaptation in Stochastic Algorithms: From Nonsmooth
  Optimization to Min-Max Problems and Beyond}.
\newblock PhD thesis, \'{E}cole polytechnique f\'{e}d\'{e}rale de Lausanne,
  2021.

\bibitem[Algoet and Cover(1988)]{Algoet1988}
P.~H. Algoet and T.~M. Cover.
\newblock Asymptotic optimality and asymptotic equipartition properties of
  log-optimum investment.
\newblock \emph{Ann. Probab.}, 16\penalty0 (2):\penalty0 876--898, 1988.

\bibitem[Alman and Williams(2021)]{Alman2021}
J.~Alman and V.~V. Williams.
\newblock A refined laser method and faster matrix multiplication.
\newblock In \emph{Proc. 2021 ACM-SIAM Symp. Discrete Algorithms (SODA)}, 2021.

\bibitem[Altepeter et~al.(2003)Altepeter, Branning, Jeffrey, Wei, Kwiat, Thew,
  O'Brien, Nielsen, and White]{Altepeter2003}
J.~B. Altepeter, D.~Branning, E.~Jeffrey, T.~C. Wei, P.~G. Kwiat, R.~T. Thew,
  J.~L. O'Brien, M.~A. Nielsen, and A.~G. White.
\newblock Ancilla-assisted quantum process tomography.
\newblock \emph{Phys. Rev. Lett.}, 90\penalty0 (19), 2003.

\bibitem[Arora et~al.(2012)Arora, Hazan, and Kale]{Arora2012}
S.~Arora, E.~Hazan, and S.~Kale.
\newblock The multiplicative weights update method: A meta-algorithm and
  applications.
\newblock \emph{Theory Comput.}, 8:\penalty0 121--164, 2012.

\bibitem[Bauschke et~al.(2017)Bauschke, Bolte, and Teboulle]{Bauschke2017}
H.~H. Bauschke, J.~Bolte, and M.~Teboulle.
\newblock A descent lemma beyond {L}ipschitz gradient continuity: first-order
  methods revisited and applications.
\newblock \emph{Math. Oper. Res.}, 42\penalty0 (2):\penalty0 330--348, 2017.

\bibitem[Blume-Kohout(2010{\natexlab{a}})]{Blume-Kohout2010}
R.~Blume-Kohout.
\newblock Hedged maximum likelihood quantum state estimation.
\newblock \emph{Phys. Rev. Lett.}, 105, 2010{\natexlab{a}}.

\bibitem[Blume-Kohout(2010{\natexlab{b}})]{Blume-Kohout2010a}
R.~Blume-Kohout.
\newblock Optimal, reliable estimation of quantum states.
\newblock \emph{New J. Phys.}, 12, 2010{\natexlab{b}}.

\bibitem[Bolduc et~al.(2017)Bolduc, Knee, Gauger, and Leach]{Bolduc2017}
E.~Bolduc, G.~C. Knee, E.~M. Gauger, and J.~Leach.
\newblock Projected gradient descent algorithms for quantum state tomography.
\newblock \emph{{npj} Quantum Inf.}, 3, 2017.

\bibitem[Breiman(1975)]{Breiman1975}
L.~Breiman.
\newblock Investment policies for expanding business optimal in a long-run
  sense.
\newblock In W.~T. Ziemba and R.~G. Vickson, editors, \emph{Stochastic
  Optimization Models in Finance}, pages 593--598. Academic Press, New York,
  NY, 1975.

\bibitem[Carderera et~al.(2021)Carderera, Besan\c{c}on, and
  Pokutta]{Carderera2021}
A.~Carderera, M.~Besan\c{c}on, and S.~Pokutta.
\newblock Simple steps are all you need: {F}rank-{W}olfe and generalized
  self-concordant functions, 2021.

\bibitem[Cesa-Bianchi et~al.(2004)Cesa-Bianchi, Conconi, and
  Gentile]{Cesa-Bianchi2004}
N.~Cesa-Bianchi, A.~Conconi, and C.~Gentile.
\newblock On the generalization ability of on-line learning algorithms.
\newblock \emph{IEEE Trans. Inf. Theory}, 50\penalty0 (9):\penalty0 2050--2057,
  2004.

\bibitem[Chambolle et~al.(2018)Chambolle, Ehrhardt, Richt\'{a}rik, and
  Sch\"{o}nlieb]{Chambolle2018}
A.~Chambolle, M.~J. Ehrhardt, P.~Richt\'{a}rik, and C.-B. Sch\"{o}nlieb.
\newblock Stochastic primal-dual hybrid gradient algorithm with arbitrary
  sampling and imaging applications.
\newblock \emph{SIAM J. Optim.}, 28\penalty0 (4):\penalty0 2783--2808, 2018.

\bibitem[Chen et~al.(2022)Chen, Huang, Li, Liu, and Selke]{Chen2022}
S.~Chen, B.~Huang, J.~Li, A.~Liu, and M.~Selke.
\newblock Tight bounds for state tomography with incoherent measurements.
\newblock 2022.
\newblock arXiv:2206.05265v1.

\bibitem[Cover(1984)]{Cover1984}
T.~M. Cover.
\newblock An algorithm for maximizing expected log investment return.
\newblock \emph{IEEE Trans. Inf. Theory}, IT-30\penalty0 (2):\penalty0
  369--373, 1984.

\bibitem[Cover(1991)]{Cover1991}
T.~M. Cover.
\newblock Universal portfolios.
\newblock \emph{Math. Financ.}, 1\penalty0 (1):\penalty0 1--29, 1991.

\bibitem[Cover and Ordentlich(1996)]{Cover1996}
T.~M. Cover and E.~Ordentlich.
\newblock Universal portfolios with side information.
\newblock \emph{IEEE Trans. Inf. Theory}, 42\penalty0 (2):\penalty0 348--363,
  1996.

\bibitem[Csisz\'{a}r and Tusn\'{a}dy(1984)]{Csiszar1984}
I.~Csisz\'{a}r and G.~Tusn\'{a}dy.
\newblock Information geometry and alternating minimization procedures.
\newblock \emph{Stat. Decis.}, \penalty0 (Supplement 1):\penalty0 205--237,
  1984.

\bibitem[Cutkosky(2019)]{Cutkosky2019}
A.~Cutkosky.
\newblock Anytime online-to-batch, optimism and acceleration.
\newblock In \emph{Proc. 36th Int. Conf. Machine Learning}, 2019.

\bibitem[Dempster et~al.(1977)Dempster, Laird, and Rubin]{Dempster1977}
A.~P. Dempster, N.~M. Laird, and D.~B. Rubin.
\newblock Maximum likelihood from incomplete data via the {EM} algorithm.
\newblock \emph{J. R. Stat. Soc., Ser. B}, 39\penalty0 (1):\penalty0 1--38,
  1977.

\bibitem[Dvurechensky et~al.(2020)Dvurechensky, Ostroukhov, Safin, Shtern, and
  Staudigl]{Dvurechensky2020}
P.~Dvurechensky, P.~Ostroukhov, K.~Safin, S.~Shtern, and M.~Staudigl.
\newblock Self-concordant analysis of {F}rank-{W}olfe algorithms.
\newblock In \emph{Proc. 37th Int. Conf. Machine Learning}, 2020.

\bibitem[Flammia et~al.(2012)Flammia, Gross, Liu, and Eisert]{Flammia2012}
S.~T. Flammia, D.~Gross, Y.-K. Liu, and J.~Eisert.
\newblock Quantum tomography via compressed sensing: Error bounds, sample
  complexity and efficient estimators.
\newblock \emph{New J. Phys.}, 14, 2012.

\bibitem[Gao and Goldfarb(2019)]{Gao2019}
W.~Gao and D.~Goldfarb.
\newblock Quasi-{N}ewton methods: superlinear convergence without line searches
  for self-concordant functions.
\newblock \emph{Optim. Methods Softw.}, 34\penalty0 (1):\penalty0 194--217,
  2019.

\bibitem[Gon\c{c}alves et~al.(2014)Gon\c{c}alves, Gomes-Ruggiero, and
  Lavor]{Goncalves2014}
D.~S. Gon\c{c}alves, M.~A. Gomes-Ruggiero, and C.~Lavor.
\newblock Global convergence of diluted iterations in maximum-likelihood
  quantum tomography.
\newblock \emph{Quantum Inf. Comput.}, 14\penalty0 (11\&12):\penalty0 966--980,
  2014.

\bibitem[Gross et~al.(2010)Gross, Liu, Flammia, Becker, and Eisert]{Gross2010}
D.~Gross, Y.-K. Liu, S.~T. Flammia, S.~Becker, and J.~Eisert.
\newblock Quantum state tomography via compressed sensing.
\newblock \emph{Phys. Rev. Lett.}, 105, 2010.

\bibitem[Gu\c{t}\v{a} et~al.(2020)Gu\c{t}\v{a}, Kahn, Kueng, and
  Tropp]{Guta2020}
M.~Gu\c{t}\v{a}, J.~Kahn, R.~Kueng, and J.~A. Tropp.
\newblock Fast state tomography with optimal error bounds.
\newblock \emph{J. Phys. A: Math. Theor.}, 53, 2020.

\bibitem[Haah et~al.(2017)Haah, Harrow, Ji, Wu, and Yu]{Haah2017}
J.~Haah, A.~W. Harrow, Z.~Ji, X.~Wu, and N.~Yu.
\newblock Sample-optimal tomography of quantum states.
\newblock \emph{IEEE Trans. Inf. Theory}, 63\penalty0 (9):\penalty0 5628--5641,
  2017.

\bibitem[H\"{a}ffner et~al.(2005)H\"{a}ffner, H\"{a}nsel, Roos, Benhelm,
  {Check-al-kar}, Chwalla, K\"{o}rber, Rapol, Riebe, Schmidt, Becher,
  G\"{u}hne, D\"{u}r, and Blatt]{Haffner2005}
H.~H\"{a}ffner, W.~H\"{a}nsel, C.~F. Roos, J.~Benhelm, D.~{Check-al-kar},
  M.~Chwalla, T.~K\"{o}rber, U.~D. Rapol, M.~Riebe, P.~O. Schmidt, C.~Becher,
  O.~G\"{u}hne, W.~D\"{u}r, and R.~Blatt.
\newblock Scalable multiparticle entanglement of trapped ions.
\newblock \emph{Nature}, 438:\penalty0 643--646, 2005.

\bibitem[He et~al.(2019)He, Harchaoui, Wang, and Song]{He2019}
N.~He, Z.~Harchaoui, Y.~Wang, and L.~Song.
\newblock Point process estimation with mirror prox algorithms.
\newblock \emph{Appl. Math. Optim.}, 2019.

\bibitem[Hradil(1997)]{Hradil1997}
Z.~Hradil.
\newblock Quantum-state estimation.
\newblock \emph{Phys. Rev. A}, 55\penalty0 (3), 1997.

\bibitem[Hradil et~al.(2004)Hradil, \v{R}eh\'{a}\v{c}ek, Fiur\'{a}\v{s}ek, and
  Je\v{z}ek]{Hradil2004}
Z.~Hradil, J.~\v{R}eh\'{a}\v{c}ek, J.~Fiur\'{a}\v{s}ek, and M.~Je\v{z}ek.
\newblock Maximum-likelihood methods in quantum mechanics.
\newblock In \emph{Quantum State Estimation}, chapter~3, pages 59--112.
  Springer, Berlin, 2004.

\bibitem[Kalai and Vempala(2002)]{Kalai2002}
A.~Kalai and S.~Vempala.
\newblock Efficient algorithms for universal portfolios.
\newblock \emph{J. Mach. Learn. Res.}, 3:\penalty0 423--440, 2002.

\bibitem[Kelly(1956)]{Kelly1956}
J.~L. Kelly, Jr.
\newblock A new interpretation of information rate.
\newblock \emph{IRE Trans. Inf. Theory}, 2\penalty0 (3):\penalty0 185--189,
  1956.

\bibitem[Koolen et~al.(2011)Koolen, Kot{\l}owski, and Warmuth]{Koolen2011}
W.~M. Koolen, W.~Kot{\l}owski, and M.~K. Warmuth.
\newblock Learning eigenvectors for free.
\newblock In \emph{Adv. Neural Information Processing Systems 24}, 2011.

\bibitem[Kueng et~al.(2017)Kueng, Rauhut, and Terstiege]{Kueng2017}
R.~Kueng, H.~Rauhut, and U.~Terstiege.
\newblock Low rank matrix recovery from rank one measurements.
\newblock \emph{Appl. Comput. Harmon. Anal.}, 42:\penalty0 88--116, 2017.

\bibitem[Li and Cevher(2019)]{Li2019a}
Y.-H. Li and V.~Cevher.
\newblock Convergence of the exponentiated gradient method with {A}rmijo line
  search.
\newblock \emph{J. Optim. Theory Appl.}, 181\penalty0 (2):\penalty0 588--607,
  2019.

\bibitem[Liu(2011{\natexlab{a}})]{Liu2011}
Y.-K. Liu.
\newblock Universal low-rank matrix recovery from {P}auli measurements.
\newblock In \emph{Adv. Neural Information Processing Systems 24},
  2011{\natexlab{a}}.

\bibitem[Liu(2011{\natexlab{b}})]{Liu2011a}
Y.-K. Liu.
\newblock Universal low-rank matrix recovery from {P}auli measurements.
\newblock 2011{\natexlab{b}}.
\newblock arXiv:1103.2816v2 [quant-ph].

\bibitem[Luo et~al.(2018)Luo, Wei, and Zheng]{Luo2018}
H.~Luo, C.-Y. Wei, and K.~Zheng.
\newblock Efficient online portfolio with logarithmic regret.
\newblock In \emph{Adv. Neural Information Processing Systems 31}, 2018.

\bibitem[Lvovsky(2004)]{Lvovsky2004}
A.~I. Lvovsky.
\newblock Iterative maximum-likelihood reconstruction in quantum homodyne
  tomography.
\newblock \emph{J. Opt. B: Quantum Semiclass. Opt.}, 6, 2004.

\bibitem[Molina-Terriza et~al.(2004)Molina-Terriza, Vaziri,
  \v{R}eh\'{a}\v{c}ek, Hradil, and Zeilinger]{MolinaTerriza2004}
G.~Molina-Terriza, A.~Vaziri, J.~\v{R}eh\'{a}\v{c}ek, Z.~Hradil, and
  A.~Zeilinger.
\newblock Triggered qutrits for quantum communication protocols.
\newblock \emph{Phys. Rev. Lett.}, 92\penalty0 (16), 2004.

\bibitem[Nesterov(2004)]{Nesterov2004}
Y.~Nesterov.
\newblock \emph{Introductory Lectures on Convex Optimization}.
\newblock Kluwer, Boston, MA, 2004.

\bibitem[Nielsen and Chuang(2010)]{Nielsen2010}
M.~A. Nielsen and I.~L. Chuang.
\newblock \emph{Quantum Computation and Quantum Information}.
\newblock Cambridge Univ. Press, Cambridge, UK, 2010.

\bibitem[O'Donnell and Wright(2016)]{ODonnell2016}
R.~O'Donnell and J.~Wright.
\newblock Efficient quantum tomography.
\newblock In \emph{Proc. 48th Annu. ACM Symp. Theory of Computing}, pages
  899--912, 2016.

\bibitem[Odor et~al.(2016)Odor, Li, Yurtsever, Hsieh, El~Halabi, Tran-Dinh, and
  Cevher]{Odor2016}
G.~Odor, Y.-H. Li, A.~Yurtsever, Y.-P. Hsieh, M.~El~Halabi, Q.~Tran-Dinh, and
  V.~Cevher.
\newblock {F}rank-{W}olfe works for non-{L}ipschitz continuous gradient
  objectives: Scalable {P}oisson phase retrieval.
\newblock In \emph{IEEE Int. Conf. Acoustics, Speech and Signal Processing},
  pages 6230--6234, 2016.

\bibitem[Opatrn\'{y} et~al.(1997)Opatrn\'{y}, Welsch, and Vogel]{Opatrny1997}
T.~Opatrn\'{y}, D.-G. Welsch, and W.~Vogel.
\newblock Least-squares inversion for density-matrix reconstruction.
\newblock \emph{Phys. Rev. A}, 56\penalty0 (3), 1997.

\bibitem[Orabona(2019)]{Orabona2019}
F.~Orabona.
\newblock A modern introduction to online learning.
\newblock 2019.
\newblock arXiv:1912.13213v1.

\bibitem[Orseau et~al.(2017)Orseau, Lattimore, and Legg]{Orseau2017}
L.~Orseau, T.~Lattimore, and S.~Legg.
\newblock Soft-{B}ayes: Prod for mixtures of experts with log-loss.
\newblock In \emph{Proc. 28th Int. Conf. Algorithmic Learning Theory}, pages
  372--399, 2017.

\bibitem[Paris and \v{R}eh\'{a}\v{c}ek(2004)]{Paris2004}
M.~Paris and J.~\v{R}eh\'{a}\v{c}ek, editors.
\newblock \emph{Quantum State Estimation}.
\newblock Springer, Berlin, 2004.

\bibitem[Quek et~al.(2021)Quek, Fort, and Ng]{Quek2021}
Y.~Quek, S.~Fort, and H.~K. Ng.
\newblock Adaptive quantum state tomography with neural networks.
\newblock \emph{{npj} Quantum Inf.}, 7, 2021.

\bibitem[Riofr\'{\i}o et~al.(2017)Riofr\'{\i}o, Gross, Flammia, Monz, Nigg,
  Blatt, and Eisert]{Riofrio2017}
C.~A. Riofr\'{\i}o, D.~Gross, S.~T. Flammia, T.~Monz, D.~Nigg, R.~Blatt, and
  J.~Eisert.
\newblock Experimental quantum compressed sensing for a seven-qubit system.
\newblock \emph{Nature Commun.}, 2017.

\bibitem[Scholten and Blume-Kohout(2018)]{Scholten2018}
T.~L. Scholten and R.~Blume-Kohout.
\newblock Behavior of maximum likelihood in quantum state tomography.
\newblock \emph{New J. Phys.}, 20, 2018.

\bibitem[Steffens et~al.(2017)Steffens, Riofr\'{\i}o, McCutcheon, Roth, Bell,
  McMillan, Tame, Rarity, and Eisert]{Steffens2017}
A.~Steffens, C.~A. Riofr\'{\i}o, W.~McCutcheon, I.~Roth, B.~A. Bell,
  A.~McMillan, M.~S. Tame, J.~G. Rarity, and J.~Eisert.
\newblock Experimentally exploring compressed sensing quantum tomography.
\newblock \emph{Quantum Sci. Tech.}, 2, 2017.

\bibitem[Tran-Dinh et~al.(2015)Tran-Dinh, Kyrillidis, and
  Cevher]{Tran-Dinh2015b}
Q.~Tran-Dinh, A.~Kyrillidis, and V.~Cevher.
\newblock Composite self-concordant minimization.
\newblock \emph{J. Mach. Learn. Res.}, 16:\penalty0 371--416, 2015.

\bibitem[Tsuda et~al.(2005)Tsuda, R\"{a}tsch, and Warmuth]{Tsuda2005}
K.~Tsuda, G.~R\"{a}tsch, and M.~K. Warmuth.
\newblock Matrix exponentiated gradient updates for on-line learning and
  {B}regman projection.
\newblock \emph{J. Mach. Learn. Res.}, 6:\penalty0 995--1018, 2005.

\bibitem[van Erven et~al.(2020)van Erven, van~der Hoeven, Kot{\l}owski, and
  Koolen]{vanErven2020}
T.~van Erven, D.~van~der Hoeven, W.~Kot{\l}owski, and W.~M. Koolen.
\newblock Open problem: Fast and optimal online portfolio selection.
\newblock In \emph{Proc. 33rd Conf. Learning Theory}, 2020.

\bibitem[\v{R}eh\'{a}\v{c}ek et~al.(2007)\v{R}eh\'{a}\v{c}ek, Hradil, Knill,
  and Lvovsky]{Rehacek2007}
J.~\v{R}eh\'{a}\v{c}ek, Z.~Hradil, E.~Knill, and A.~I. Lvovsky.
\newblock Diluted maximum-likelihood algorithm for quantum tomography.
\newblock \emph{Phys. Rev. A}, 75, 2007.

\bibitem[Warmuth and Kuzmin(2010)]{Warmuth2010}
M.~K. Warmuth and D.~Kuzmin.
\newblock {B}ayesian generalized probability calculus for density matrices.
\newblock \emph{Mach. Learn.}, 78:\penalty0 63--101, 2010.

\bibitem[Wilde(2019)]{Wilde2019}
M.~Wilde.
\newblock From classical to quantum {S}hannon theory.
\newblock 2019.
\newblock arXiv:1106.1445v8.

\bibitem[Yang et~al.(2020)Yang, Jiang, Zhang, and Sun]{Yang2020}
F.~Yang, J.~Jiang, J.~Zhang, and X.~Sun.
\newblock Revisiting online quantum state learning.
\newblock In \emph{Proc. AAAI Conf. Artificial Intelligence}, 2020.

\bibitem[Youssry et~al.(2019)Youssry, Ferrie, and Tomamichel]{Youssry2019}
A.~Youssry, C.~Ferrie, and M.~Tomamichel.
\newblock Efficient online quantum state estimation using a
  matrix-exponentiated gradient method.
\newblock \emph{New J. Phys.}, 21\penalty0 (033006), 2019.

\bibitem[Zhao and Freund(2020)]{Zhao2020}
R.~Zhao and R.~M. Freund.
\newblock Analysis of the {F}rank-{W}olfe method for
  logarithmically-homogeneous barriers, with an extension, 2020.

\bibitem[Zhou et~al.(2017)Zhou, Gao, and Goldfarb]{Zhou2017}
C.~Zhou, W.~Gao, and D.~Goldfarb.
\newblock Stocahstic adaptive quasi-{N}ewton methods for minimizing expected
  values.
\newblock In \emph{Proc. 34th Int. Conf. Machine Learning}, 2017.

\bibitem[Zimmert et~al.(2022)Zimmert, Agarwal, and Kale]{Zimmert2022}
J.~Zimmert, N.~Agarwal, and S.~Kale.
\newblock Pushing the efficiency-regret {P}areto frontier for online learning
  of portfolios and quantum states.
\newblock 2022.
\newblock arXiv:2202.02765v1.

\end{thebibliography}

\newpage
\appendix

\section{Supplementary Numerical Results} \label{app_numerical}

\begin{figure}[!htb]

\begin{minipage}{.48\textwidth}
\includegraphics[scale=.4]{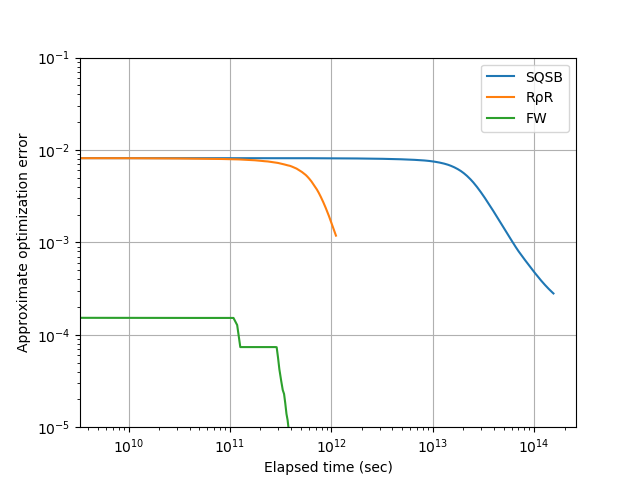}
\caption{Approximate optimization errors in function value of Stochastic Q-Soft-Bayes (SQSB),  \( R \rho R \) \deleted{(RrhoR)},  and Monotonous Frank-Wolfe (FW). 
Notice that the x-axis is the elapsed time instead of number of epochs.}
\label{fig_error_time}
\end{minipage}\hfill
\begin{minipage}{.48\textwidth}
\includegraphics[scale=.4]{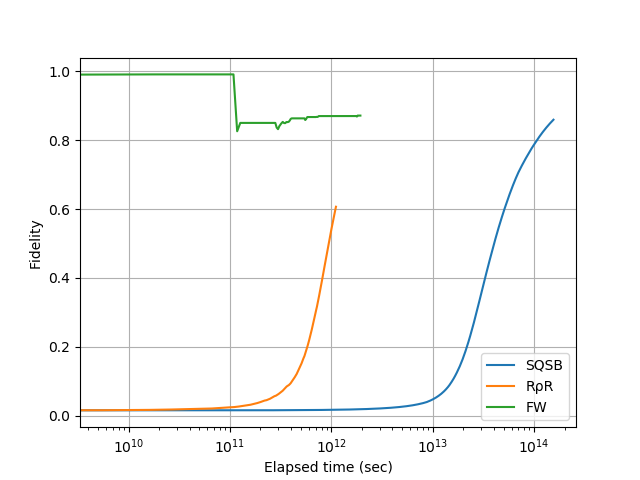}
\caption{Fidelity values of the iterates and the W state achieved by Stochastic Q-Soft-Bayes (SQSB),  \( R \rho R \) \deleted{(RrhoR)},  and Monotonous Frank-Wolfe (FW).
Notice the x-axis is the elapsed time instead of number of epochs.}
\label{fig_fidelity_time}
\end{minipage}

\end{figure}

As the three algorithms we compare in Section \ref{sec_numerical} have different per-iteration computational complexities, it can be misleading to evaluate their convergence speeds only in terms of the number of epochs. 
Figure \ref{fig_error_time} and Figure \ref{fig_fidelity_time} are plotted in terms of the real elapsed time on a \added{server}
with \added{a Intel Xeon Gold 5218 CPU of 2.30GHz and 131621512kB memory. }
The algorithms are implemented with the programming language Julia\added{ and the Math Kernel Library (MKL)}. 
\added{The number of threads for BLAS is set to 8.}

\added{It is easy to evaluate the per-iteration time complexities. }
\added{Every} iteration of Stochastic Q-Soft-Bayes involves matrix exponential and matrix logarithm computations, \added{so} \added{its per-iteration time complexity is $O ( D^3 )$.}
\added{The time to compute the batch gradient is $O ( N D ^ 2 )$.
Therefore, the per-iteration time complexity of Monotonous Frank-Wolfe is $O ( N D ^ 2 )$, dominated by the time of computing the batch gradient; 
the per-iteration time complexity of $R \rho R$ is $O ( N D ^ 2 + D ^ {\omega'} )$, where the $D^{\omega'}$ term comes from the time of computing matrix multiplications.
Here, one iteration of Stochastic Q-Soft-Bayes corresponds to $1 / N$ ``iteration'' of $R \rho R$ and Monotonous Frank-Wolfe.
However, a theoretical comparison as in Section \ref{sec_comparison} is impossible. 
$R \rho R$ does not always converge and hence lacks an iteration complexity guarantee; 
the dependence of the iteration complexity of Monotonous Frank-Wolfe on $N$ and $D$ is unclear. }

\added{The numerical results show that in terms of the elapsed time, Stochastic Q-Soft-Bayes is much slower than the other two algorithms, showing a large room for improvement to compete with practical--yet theoretically non-rigorous---algorithms.
Section \ref{sec_faster} discusses possible approaches to develop faster algorithms.
It should be emphasized again that the elapsed time highly depends on the hardware and implementation details.
}



\section{Proof of Proposition \ref{prop_small_trace}} \label{sec_small_trace}

By the Golden-Thompson inequality,  we write 
\begin{align*}
\tr ( W_{t + 1} ) & \leq \tr W_t \left( ( 1 - \eta ) I + \eta \frac{A_t}{\tr ( A_t \rho_t )} \right) \\
& = ( 1 - \eta ) \tr ( W_t ) + \eta \frac{\tr ( A_t W_t )}{ \tr ( A_t \rho_t ) }  \\
& = ( 1 - \eta ) \tr ( W_t ) + \eta \tr ( W_t ) \\
& = \tr ( W_t ) ,  
\end{align*}
showing that \( \tr ( W_{t + 1} ) \leq \tr ( W_1 ) = 1 \). 

\section{Proof of Theorem \ref{thm_qsb}} \label{app_qsb}

For convenience,  we define \( \braketH{A,  B} = \tr ( A B ) \) for any two Hermitian matrices \( A \) and \( B \). 
Indeed,  the inner product \( \braketH{ \cdot,  \cdot } \) is the Hilbert-Schmidt inner product in literature. 
The following lemma---a reverse Jensen's inequality for the logarithmic function---is a quantum extension of Lemma 4 in \cite{Orseau2017}. 
We omit the proof as it is simply the original proof in \cite{Orseau2017} applied to the eigenvalues of the matrices involved. 

\begin{lemma} \label{lem_reverse_jensen}
Let \( \eta \in ( 0,  1 ) \),  \( \rho \in \mathcal{D} \),  and \( X \in \mathbb{C}^{D \times D} \) be a Hermitian positive semi-definite matrix. 
We have 
\begin{align*}
\log \BraketH{X,  \rho} & \leq \frac{1}{\eta} \braketH{ \log \left( ( 1 - \eta ) I + \eta X \right),  \rho } + \tr \left( \log \left( I + \frac{\eta}{1 - \eta} X \right) \right) . 
\end{align*}
\end{lemma}

Now we are able to prove Theorem \ref{thm_qsb}. 

\begin{proof}[Proof of Theorem \ref{thm_qsb}]

By Lemma \ref{lem_reverse_jensen} and the definition that \( \rho_t = W_t / \tr ( W_t ) \),  we write 
\begin{align*}
& - \log \braketH{ A_t,  \rho_t } - \left( - \log \braketH{ A_t,  \rho } \right) \\
& \quad = \log \BraketH{ \frac{A_t}{\braketH{ A_t,  \rho_t }},  \rho } \\
& \quad \leq \frac{1}{\eta} \BraketH{ \log \left( ( 1 - \eta ) I + \eta \frac{ A_t }{\braketH{ A_t,  \rho_t }} \right),  \rho } + \tr \left( \log \left( I + \frac{\eta}{1 - \eta} \frac{A_t}{\braketH{ A_t,  \rho_t }} \right) \right) \\
& \quad = \frac{1}{\eta} \BraketH{ \log \left( ( 1 - \eta ) I + \eta \tr ( W_t ) \frac{ A_t }{\braketH{ A_t,  W_t }} \right),  \rho } + \\ 
& \quad \quad \, \, \tr \left( \log \left( I + \frac{\eta \tr ( W_t )}{1 - \eta} \frac{A_t}{\braketH{ A_t,  W_t }} \right) \right) , 
\end{align*}
for all density matrix \( \rho \) such that \( \braketH{ A_t,  \rho } > 0 \). 
Recall that Proposition \ref{prop_small_trace} ensures \( \tr ( W_t ) \leq 1 \). 
Then,  by the iteration rule of Q-Soft-Bayes,  we write 
\begin{align*}
& - \log \braketH{ A_t,  \rho_t }) - \left( - \log \braketH{ A_t,  \rho } \right) \\
& \quad \leq \frac{1}{\eta} \BraketH{ \log \left( ( 1 - \eta ) I + \eta \frac{ A_t }{\braketH{ A_t,  W_t }} \right),  \rho } + \tr \left( \log \left( I + \frac{\eta }{1 - \eta} \frac{A_t}{\braketH{ A_t,  W_t }} \right) \right) \\ 
& \quad = \frac{1}{\eta} \BraketH{ \log \left( ( 1 - \eta ) I + \eta \frac{ A_t }{\braketH{ A_t,  W_t }} \right),  \rho } + \\
& \quad \quad \, \, \tr \left( \log \left( ( 1 - \eta ) I + \eta \frac{A_t}{\braketH{ A_t,  W_t }} \right) - \log ( 1 - \eta ) I \right) \\
& \quad = \frac{1}{\eta} \BraketH{ \log \left( W_{t + 1} \right) - \log \left( W_t \right),  \rho } + \tr \left( \log \left( W_{t + 1} \right) - \log \left( W_t \right) \right) - D \log ( 1 - \eta ) . 
\end{align*}
By Lemma 13 in \cite{Orseau2017},  we have 
\[ 
- \log ( 1 - \eta ) \leq \overline{\eta} \coloneqq \eta / ( 1 - \eta ).
\] 
Hence,  we get
\begin{align*}
& - \log \braketH{ A_t,  \rho_t } - \left( - \log \braketH{ A_t,  \rho } \right) \\
& \quad \leq \frac{1}{\eta} \BraketH{ \log \left( W_{t + 1} \right) - \log \left( W_t \right),  \rho } + \tr \left( \log \left( W_{t + 1} \right) - \log \left( W_t \right) \right) + \overline{\eta} D . 
\end{align*}
Telescoping the inequality above,  we write 
\begin{align*}
& \sum_{t = 1}^T \left( - \log \braketH{ A_t,  \rho_t } \right) - \sum_{t = 1}^T \left( - \log \braketH{ A_t,  \rho } \right) \\
& \quad \leq \frac{1}{\eta} \BraketH{ \log \left( W_{T + 1} \right) - \log \left( W_1 \right),  \rho } + \tr \left( \log \left( W_{T + 1} \right) - \log \left( W_1 \right) \right) + \overline{\eta} T D ,  
\end{align*}
which indeed provides an upper bound on \( R_T \) as it holds for all possible \( ( A_t )_{t \in \mathbb{N}} \). 
Notice that \( W_1 = \rho_1 = I / D \) and again,  by Proposition \ref{prop_small_trace},  we have 
\[
\log \left( W_{T + 1} \right) = \log \left( \rho_{T + 1} \tr ( W_{T + 1} ) \right) \leq \log \left( \rho_{T + 1} \right) . 
\]
Then,  we write 
\begin{align*}
R_T & \leq \frac{1}{\eta} \BraketH{ \log \left( \rho_{T + 1} \right) - \log \left( \rho_1 \right),  \rho } + \tr \left( \log \left( \rho_{T + 1} \right) - \log \left( \rho_1 \right) \right) + \overline{\eta} T D \\
& = \frac{1}{\eta} \BraketH{ \log \left( \rho_{T + 1} \right) - \log \left( \rho_1 \right),  \rho } + D \BraketH{ \log \left( \rho_{T + 1} \right) - \log \left( \rho_1 \right),  \rho_1 } + \overline{\eta} T D \\
& = \frac{1}{\eta} \left( S ( \rho \Vert \rho_1 ) - S ( \rho \Vert \rho_{T + 1} ) \right) - D S ( \rho_1 \Vert \rho_{T + 1} ) + \overline{\eta} T D ,  
\end{align*}
where \( S \) denotes the quantum relative entropy,  i.e.,  
\[
S ( \rho \Vert \sigma ) \coloneqq \braketH{ \log \left( \rho \right) - \log \left( \sigma \right) ,  \rho } . 
\]
It is known that the quantum relative entropy is non-negative and \( S ( \rho \Vert I / D ) \leq \log D \) for all \( \rho \in \mathcal{D} \). 
Therefore,  we obtain
\begin{align*}
R_T \leq \frac{1}{\eta} \log D + \overline{\eta} T D = \frac{1}{\overline{\eta}} \log D + \overline{\eta} T D + \log D . 
\end{align*}
The regret bound is optimized when 
\[ 
\overline{\eta} = \sqrt{ \frac{\log D}{T D} } . 
\] 

\end{proof}


%
%
%
%

\end{document}